\documentclass[runningheads]{llncs}
\usepackage{graphicx}
\usepackage{comment}
\usepackage{amsmath} 
\usepackage{color}

\usepackage[misc]{ifsym}

\usepackage{algorithm}
\usepackage{algorithmic}
\usepackage{amsmath}
\usepackage{makecell,multirow,diagbox}
\usepackage{subfigure}
\usepackage{graphicx}
\usepackage{array}
\usepackage{booktabs} 

\begin{document}
\pagestyle{headings}
\mainmatter
\def\ECCVSubNumber{1811}  

\title{iffDetector: Inference-aware Feature Filtering for Object Detection} 


\titlerunning{iffDetector}
%

\author{Mingyuan Mao\textsuperscript{1,${\dag}$} \and
Yuxin Tian\textsuperscript{1,${\dag}$} \and Baochang Zhang\textsuperscript{1,*} \and Qixiang Ye\textsuperscript{2} \and Wanquan Liu\textsuperscript{3} \and Guodong Guo\textsuperscript{4,5} \and David Doermann\textsuperscript{6}}

\authorrunning{Mao, M. and Tian, Y. et al.}
%
\institute{\textsuperscript{1}Beihang University, Beijing, China\\
\textsuperscript{2}University of Chinese Academy of Sciences, Beijing, China\\
\textsuperscript{3}Curtin University, Perth, Australia\\
\textsuperscript{4}Institute of Deep Learning, Baidu Research, Beijing, China\\
\textsuperscript{5}National Engineering Laboratory for Deep Learning Technology and Application\\
\textsuperscript{6}University at Buffalo, Buffalo, USA\\
\textsuperscript{*}Corresponding author, email: bczhang@buaa.edu.cn\\
\textsuperscript{${\dag}$} Co-first author\\}
\maketitle
\begin{abstract}
Modern CNN-based object detectors focus on feature configuration during training but often ignore feature optimization during inference. In this paper, we propose a new feature optimization approach to enhance features and suppress background noise in both the training and inference stages. We introduce a generic Inference-aware Feature Filtering (IFF) module that can easily be combined with modern detectors, resulting in our iffDetector. Unlike conventional open-loop feature calculation approaches without feedback, the IFF module performs closed-loop optimization by leveraging high-level semantics to enhance the convolutional features.  By applying Fourier transform analysis, we demonstrate that the IFF module acts as a negative feedback that theoretically guarantees the stability of feature learning. IFF can be fused with CNN-based object detectors in a plug-and-play manner with  negligible computational cost overhead. Experiments on the PASCAL VOC and MS COCO datasets demonstrate that our iffDetector consistently outperforms state-of-the-art methods by significant margins\footnote{The test code and model are anonymously available in https://github.com/anonymous2020new/iffDetector }.

\keywords{Object Detection, iffDetector, IFF, Negative Feedback}
\end{abstract}

\section{Introduction and Related Work}

We have recently witnessed the success of visual object detection thanks to the unprecedented representation capacity of convolutional neural networks (CNNs) \cite{RCNN14,FastRCNN15,FasterRCNN15,YOLO16,Redmon_2019,FPN17,SSD16}. In~\cite{Survey2019}, various taxonomies have been used to categorize the large number of CNN-based object detection methods, including  one-stage~\cite{FasterRCNN15} vs. two-stage~\cite{FocalLoss17}, single-scale features ~\cite{FasterRCNN15} vs. feature pyramid networks~\cite{FPN17}, and handcrafted networks~\cite{SSD16} vs. network architecture search~\cite{NAS-FPN2019}. 

To  explore  the representation capability of CNNs,  these detectors widely adopt a one-path network architecture which focuses on feature configuration during training \cite{Fu2016dssd,IoU-Net18,he2018bounding,zhang2019freeanchor,CornerNet2018,CenterNet2019,tian2019fcos,kong2019foveabox} but unfortunately ignore feature optimization during inference. They take the feature maps extracted by the backbone as the input, apply various convolution operations, and acquire the final predictions at the end of the network \cite{ResNet16,Survey2019,FasterRCNN15,Liu_2018_ECCV,FPN17}. For instance, the
Feature Pyramid Network (FPN)~\cite{FPN17} and top-down modulation~\cite{TDM16} are designed to fuse multi-scale convolutional features, which improve the standard feature extraction pyramid by adding a second pyramid that takes the high level features from the first pyramid and passes them down to lower layers. 
This allows features at each level to have access to both higher and lower level features. The feature selective anchor-free  (FSAF)~\cite{zhu2019feature} approach dynamically selects the most suitable level of feature for each instance based on the instance appearance, so as to address the scale variation problem.
To strengthen the deep features learned from lightweight CNN models, RFB Net~\cite{Liu_2018_ECCV} has a novel module called the Receptive Field Block (RFB) which makes use of a multi-branch pooling with varying kernels corresponding to the receptive fields of different sizes, and reshapes them to generate enforced representations. 

To facilitate object localization, many CNN-based detectors leverage anchor boxes at multiple scales and aspect ratios as reference points to address optimal feature-object matching~\cite{MetaAnchor2018,RefineDet2018,GuidedAnchoring,UnitBox2016}. The MetaAnchor~\cite{MetaAnchor2018} approach learns to optimize anchors from arbitrary customized prior boxes with a sub-network. GuidedAnchoring~\cite{GuidedAnchoring} leverages semantic features to guide the prediction of anchors while replacing dense anchors with predicted anchors.  The state-of-the-art FreeAnchor~\cite{zhang2019freeanchor} optimizes the anchor matching process by formulating a maximum likelihood estimation (MLE) procedure, which infers the most representative anchor from a ``bag" of anchors for each object. 
Considering that anchor boxes introduce many hyper-parameters and design choices that make the network hard to train, anchor-free frameworks were recently proposed for object detection~\cite{CornerNet2018,CenterNet2019,tian2019fcos,kong2019foveabox}.
CornerNet~\cite{CornerNet2018} detects an object as a pair of key points on the top-left and bottom-right corners of a bounding box without anchors, which greatly simplifies the output of the network and eliminates the requirement for designing anchor boxes. Like CornerNet, CenterNet~\cite{CenterNet2019} and FoveaBox~\cite{kong2019foveabox} were introduced with  more efficiency for object detection.

Unfortunately, CNN-based detectors are still challenged by  complex scenes where the targets and background become ambiguous.  This may be caused by the feature representation that is not optimized according to the inference. To the best of our knowledge, there are no systematic approaches for feature optimization during inference, which inhibits the consistent enhancement of detectors. 
An inference module in \cite{cascade18} was introduced to efficiently distinguish the target objects from the background by adding a cascade structure at inference. FreeAnchor \cite{zhang2019freeanchor}  leverages spatial inference to enforce feature representation and feature-object matching. These inference modules however were not investigated to optimize the feature representation in a theoretical framework in these detectors. 

\begin{figure}
\centering
\includegraphics[height=5.0cm]{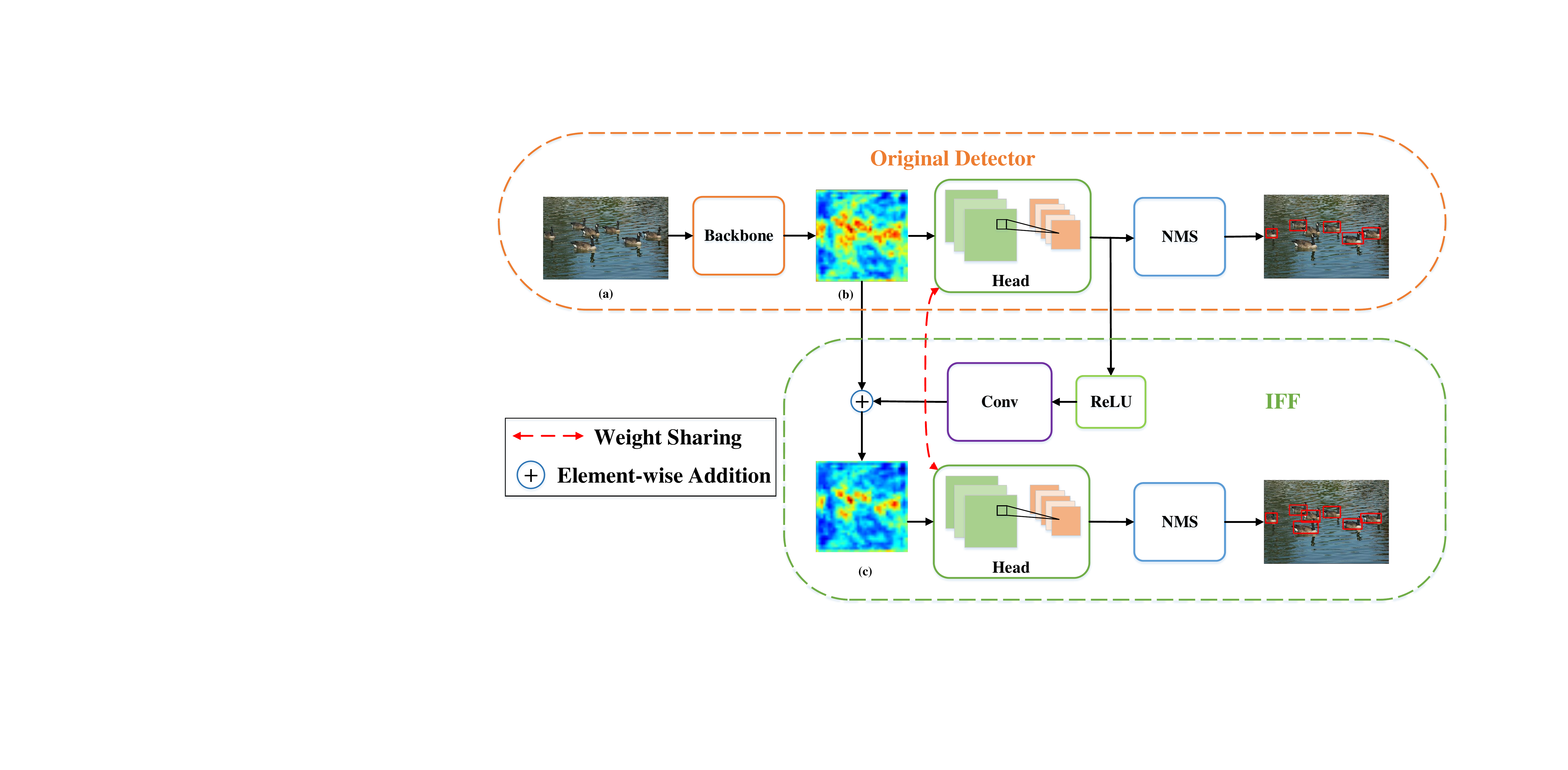}
\caption{Feature map denoising: (a) input image; (b) original feature map extracted by the backbone; (c) feature map optimized by IFF. (b) and (c)  have more than one channel, and we visualize them by upsampling the feature maps to $300 \times 300$, adding all channels together and converting them to heatmaps, where the red regions represent stronger activation.}
\label{extra1}
\end{figure}

In this paper, we propose an Inference-aware Feature Filtering (IFF) approach, which optimizes feature learning in a theoretical framework by introducing a feedback architecture in detection networks. 
IFF differs from the existing inference frameworks in that it incorporates the conventional inference module into our filtering module to refine the feature map, instead doing it just for the final detection.
Prior research shows that activated regions on a feature map are more likely to contain targets. In reality, however, the feature map extracted by the backbone may not be accurate. On the original feature map, some parts of the background are falsely activated (see Fig.~\ref{extra1}(b)), which may lead to detection errors. By applying IFF to the original feature map, we can enhance the correctly activated regions, making them more intense, while at the same time depressing the background (see Fig.~\ref{extra1}(c)). This results in more accurate object predictions. The IFF procedure, as shown in Fig.~\ref{extra1} and detailed in Sections $2$ and $3$, can be performed repeatedly. 

Approximating a negative feedback system in the detection network, we prove the stability of the IFF-based object detection system. This leads to a new and generic inference model based on light-weight convolutional layers to improve the detection performance for both anchor-based and anchor-free detectors.	By evaluating various CNN-based object detectors  including SSD~\cite{SSD16}, YOLO~\cite{YOLO9000}, FreeAnchor~\cite{zhang2019freeanchor}, FoveaBox~\cite{kong2019foveabox}, and CenterNet~\cite{CenterNet2019}, on PASCAL VOC2007+2012 and COCO 2017, we demonstrate the general applicability and effectiveness of IFF. The contributions of this paper are summarized as follows:
\begin{enumerate}
    \item  We present a generic Inference-aware Feature Filtering (IFF) module, which can be easily mounted on both anchor-based and anchor-free detectors, to optimize CNN features during inference.
    \item  We provide an explainable investigation into the proposed IFF, which theoretically guarantees the feature learning stability of the CNN-based detectors.
    \item We achieve a significant performance gain over various existing detectors (including the state-of-the-art) on average precision (AP) on the COCO 2017 data set, with a  negligible increase in computational cost.
\end{enumerate}  
 
\begin{figure}
\centering
\includegraphics[height=4.3cm]{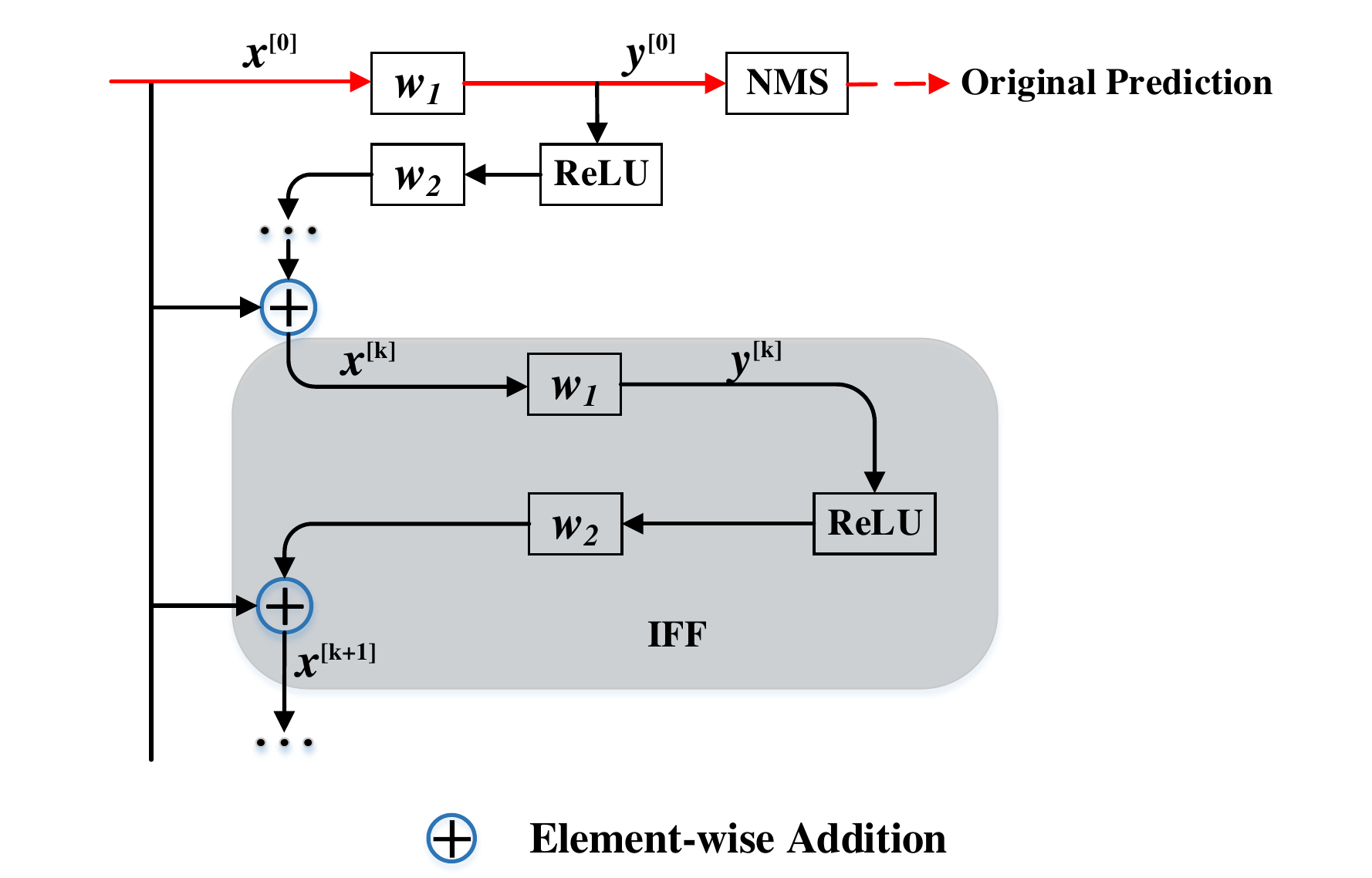}
\caption{Inference-aware feature filtering module: $w_1$ denotes the convolutional filters in the head of the original detector; $w_2$ denotes the convolutional filters we add to the system; $x^{[k]}$ and $y^{[k]}$ represent the feature map and the output in the $k^{th}$ forward propagation respectively; the addition operation introduces  a feedback  from the detection output into the feature learning to achieve a negative feedback (control theory \cite{1104893}).}
\label{architecture}
\end{figure}

\section{Inference-aware Feature Filtering}
	
In this section, we first detail the mechanism of the proposed filtering module and its advantages, then provide the framework for integrating IFF into a detector and explain how it works. 
 
\subsection{The Filtering Module}

Our IFF is designed to highlight the features that can actively respond to  the detection output (location and classification) and it provides a new and generic inference method for object detection. Our filtering module introduces a feedback link from the detection output to filter undesirable features. Suppose that for each input feature map $x^{[0]}$, there exists an ideal feature map $n$ that can lead to the ideal detection results, and additive noise $\delta$, such that
\begin{equation}
	\label{bczhangdelta}
	\begin{aligned}
	x^{[0]} = n +\delta.
	\end{aligned}
\end{equation}

To mitigate the noise effect on the features (Fig.~\ref{architecture}), we incorporate the output of the detection into our inference-aware feature filtering (IFF) module and have
\begin{equation}
	\begin{aligned}
	x^{[k+1]} = x^{[0]} + w_2 \otimes ReLU(y^{[k]}),
	\end{aligned}
	\label{conv-1}
\end{equation}
and
\begin{equation}
	\begin{aligned}
	y^{[k]} = w_1 \otimes x^{[k]},
	\end{aligned}
	\label{conv-2}
\end{equation}

\noindent
where $x^{[k]}$ and $y^{[k]}$ represent the feature map and the output in the $k^{th}$ forward propagation; $w_1$ denotes the forward convolution filters, $w_2$ are the new convolutional filters added, and $\otimes$ denotes the convolution operation. All the ReLU functions in this paper are Leaky ReLU. 

Existing detectors without feedback only implement the forward propagation once at inference, which is the red path in Fig.~\ref{architecture}. The corresponding output\footnote{$Prediction = NMS(y^{[0]})$, where NMS is short for Non-Maximum Suppression.} is
\begin{equation}
	\begin{aligned}
	y^{[0]} = w_1 \otimes x^{[0]}.
	\end{aligned}
\end{equation}

\noindent
Note that we not only regard $y^{[0]}$ as heuristic information in IFF, but also use it for the final prediction. Thus, we introduce negative feedback architectures and obtain $\{(x^{[0]}, y^{[0]}), (x^{[1]}, y^{[1]}), \cdots (x^{[k]}, y^{[k]}), \cdots (x^{[M_I]}, y^{[M_I]})\}$ based on Fig.~\ref{architecture}. The maximum value $M_I$ is flexible and  can be adjusted as necessary. $y^{[M_I]}$ is used to make the final prediction.

We explore the relationship between $x^{[k]}$ and $x^{[k+1]}$ for feature filtering, and show that the feature capacity is visually enhanced by our generic module in Fig.~\ref{architecture}.  In addition, the efficacy is theoretically validated based on the feedback mechanism in  control theory, where the addition operation introduces feedback from the detection output into the feature learning to achieve a negative feedback (control theory \cite{1104893}). The detailed  proof is provided in Section $3$, verifying that such a feedback system is stable in terms of filtering the noise $\delta$. Briefly, the noise effect on $x^{[k+1]}$ is reduced in comparison to the effect on $x^{[k]}$. The framework of integrating IFF into a detector is provided in Section $2.2$. This feedback structure has two main advantages, denoising and exploring high level semantics:
\begin{enumerate}
    \item \textbf{Denoising.} Input images often have large variations of  lighting and object occlusion, which leads to a ``polluted" feature map generated by the CNN backbone. As a result, a prediction based on such a feature map can be somewhat ambiguous. Introducing a feedback structure to the CNN updates the open-loop feature calculation to a closed-loop feature optimization. Based on control theory, such an architecture makes the feature map more robust to noise. This negative feedback reuses the heuristic information in the original prediction, denoises the feature map, and makes the correctly activated regions on it more intensive. More details can be found in Section 3.2.
    \item \textbf{Exploring high-level semantics.} It is rational to infer that classification and bounding boxes have a coupling relationship. For example, the bounding boxes for cars and people are generally different, in particular with respect to their aspect ratios. To this end, the feedback structure helps merge the location and classification, promoting a detector to learn the coupling relationships.
\end{enumerate}

\subsection{Our iffDetector}

The IFF module is based on a set of light-weight convolution operations that can lead to enhanced features. To clarify how IFF works, we update YOLOv2 \cite{YOLO9000} to iffDetector (see Fig.~\ref{filtering-module}). IFF introduces a unique feedback strategy  for the inference, forming a cascade architecture. Existing detectors without feedback use the one-path architecture at inference, which is the red path in  Fig.~\ref{filtering-module}. For IFF an extra convolution layer $w_2$ is added to merge the prediction and the 
\begin{figure}
\centering
\includegraphics[height=5.6cm]{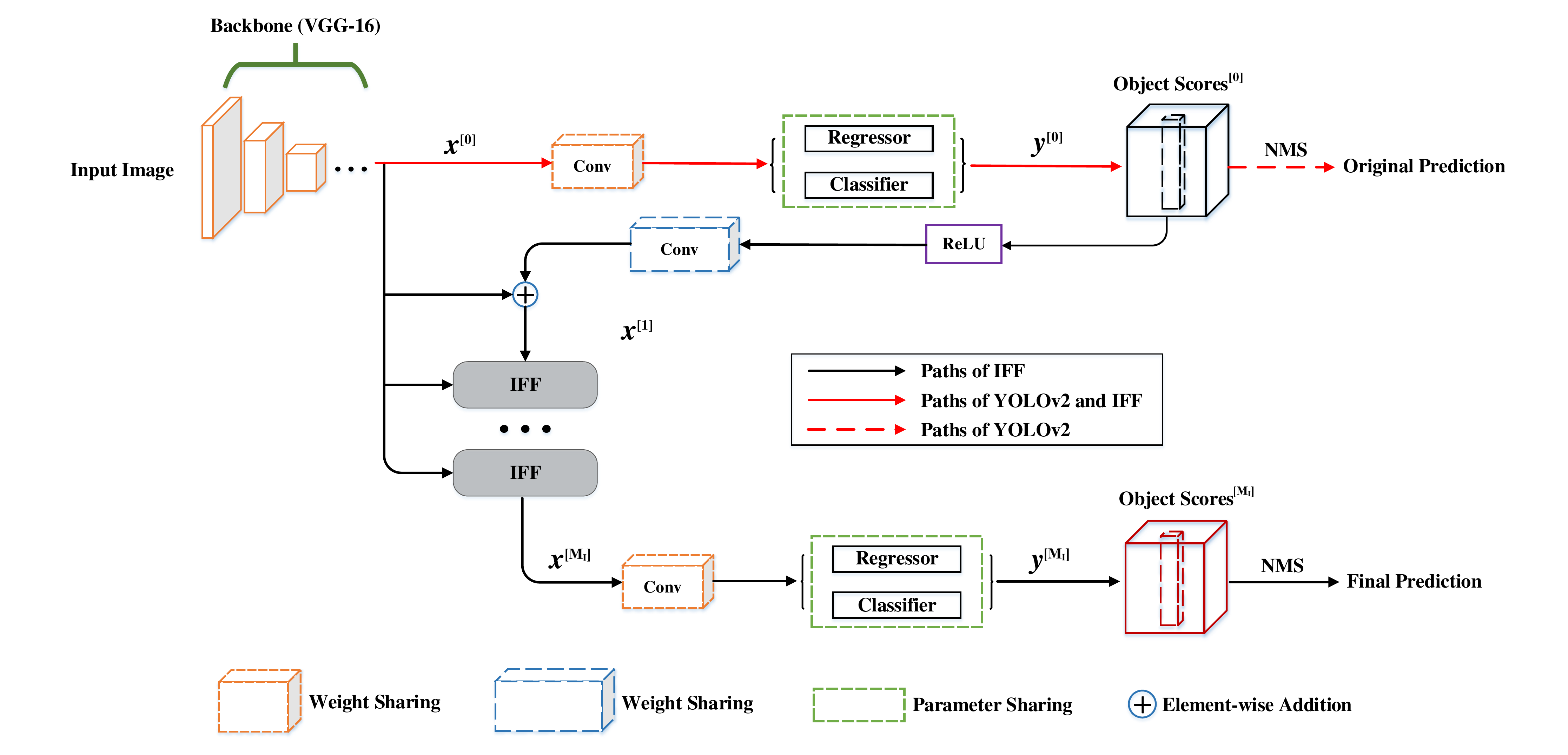}
\caption{The architecture of iffDetector: $\ast ^{[k]}$ represents the feature in the $k^{th}$ filtering module, and each IFF module is shown in Fig.~\ref{architecture}; $M_I$ is the number of the filtering modules we use; $x^{[0]}$ is the original feature map. Note that the head of YOLOv2 \cite{YOLO9000} only has one convolution layer.}
\label{filtering-module}
\end{figure}
feature map. The original object scores are not regarded as the final predictions. Instead we apply the convolution to the object scores, and then add the convolution results to the original feature map $x^{[0]}$, generating a new feature map with better features. As mentioned in Section $2.1$, this process can be implemented several times, which means the depth of the cascade architecture is flexible. The final feature map we acquire generates the final prediction. Unlike YOLOv2 \cite{YOLO9000}, most existing detectors predict bounding boxes at different scales, which can also be updated to iffDetectors by simply adding the filtering modules independently across every scale similar to Fig.~\ref{filtering-module}.  $M_I$ is the number of the filtering modules used. We then apply NMS on $y^{[M_I]}$ to obtain the final prediction. Note that a larger $M_I$ does not always guarantee a better result because of training inefficiency. Our experiments show that setting $M_I$ to 1 or 2 is appropriate.
	
\textbf{Generalizability of Inference-aware Feature Filtering. }
The filtering module can be applied to various CNN-based detectors including anchor-based and anchor-free frameworks. Without loss of generality our filtering module is simple and effective. The size of the convolution filters we design is flexible. We use $1 \times 1$ filters for the backbones darknet19 and VGG-16 and $3 \times 3$ filters for ResNet-50-FPN and ResNeXt-64x4d-101-FPN, with less than a 2\% increase in the number of parameters.

Based on the above description of IFF, we develop a universal detector training procedure as described in Algorithm \ref{alg:training}.
	\begin{algorithm}[tb]
		\caption{Training iffDetector based on IFF}
		\label{alg:training}
		\begin{algorithmic}
			\STATE {\bfseries Input:}\\
			$X$: Input images.\\
			$G$: Labels of the training data.\\
		    {$M_I$}: The number of the filtering modules.\\
			\STATE {\bfseries Output:} $\theta$ (e.g., $w_1$, $w_2$); Detection parameters.
			\STATE $\theta \gets$ initialize network parameters.
			\STATE {Obtain original feature map $x^{[0]}$.}
			\REPEAT
			\FOR{$k=0$ : $M_I-1$}
			\STATE \textbf{Forward propagation:} Obtain $y^{[k]}$: $y^{[k]} \gets w_1 \otimes x^{[k]}$.
			\STATE \textbf{Denoise the feature map:} $x^{[k+1]} \gets x^{[0]} + w_2 \otimes ReLU(y^{[k]})]$.
			\ENDFOR
			\STATE \textbf{Calculate the loss:} $L(\theta; G)$.
			\STATE \textbf{Backward propagation:} $\theta \gets \theta - \bigtriangledown L(\theta; G)$
			\UNTIL{The convergence of the loss}
			\STATE {\bfseries Return} $\theta$
		\end{algorithmic}
	\end{algorithm}

\section{Theoretical Investigation}
	
To analyze our iffDetector, we first apply the Fourier transform on the inference module. We then explain how our module affects feature learning from a theoretical point of view.
	
\subsection{Properties of the Detector Based on the Fourier Transform}
To discuss the properties of the detector we  apply the Fourier transform on Eqs. \ref{conv-1} and \ref{conv-2}, and have
\begin{equation}
	\begin{aligned}
	Y^{[k]} = W_1 \ast X^{[k]},
	\end{aligned}
	\label{frequency-1}
\end{equation}
\noindent
which represents the linear case for the Fourier transform, and
\begin{equation}
	\begin{aligned}
	X^{[k+1]} &= X^{[0]} + W_2 \ast \mathcal{F} (h(y^{[k]}))\\
	        &= X^{[0]} + W_2 \ast H(Y^{[k]}),
	\end{aligned}
	\label{frequency-2}
\end{equation}
\noindent
where $\ast$ is the element-wise multiplication operator, $h(x)$ is the  ReLU function, $\mathcal{F}$ denotes the Fourier transform, and each capitalized letter denotes  the Fourier transform of the corresponding variable.
	
Eq.~\ref{frequency-2} represents the nonlinear case when using Fourier transform to analyze the detector. This  remains a challenge in terms of interpretability for deep learning, because of the nonlinear ReLU function. Theorem 1 shows that the Leaky ReLU function can reduce the energy of a signal's Fourier transform and at the same time can be further used to prove the feature learning stability of IFF (Section 3.2).

\begin{theorem}
	For any input feature map $x$, it satisfies
	\begin{equation}
	\begin{aligned}
	\|\mathcal{F}(h(x))\|\leq\|\mathcal{F}(x)\|,
	\end{aligned}
	\label{theorem1}
	\end{equation}
	where  $\|\cdot\|$ denotes the Frobenius norm of a matrix.
\end{theorem}

\begin{proof}
	\label{proof2}
	The Leaky ReLU is defined as
	\begin{equation}
	\begin{aligned}
	h(x)=
	\left\{
	\begin{array}{lr}
	x,& x>0,\\
	\lambda x,& x\leq0,
	\end{array}
	\right.\nonumber
	\end{aligned}
	\end{equation}
	where $\lambda\in(0, 1)$. If $x$ is a matrix of size $W \times H$, we use $x_i$ to denote its elements, where $i\in{[0, WH-1]}$. According to the definition of the Leaky ReLU function, it is obvious that
	\begin{equation}
	\begin{aligned}
	\sum_{i=0}^{WH-1}|h(x_i)|^2\leq\sum_{i=0}^{WH-1}|x_i|^2.
	\end{aligned}
	\end{equation}
	From Lemma 1 in the appendix, we can obtain
	\begin{equation}
	\begin{aligned}
	\sum_{i=0}^{WH-1}|h(x_i)|^2 = \frac{1}{WH}\sum_{k=0}^{WH-1}|H(X_k)|^2,
	\end{aligned}
	\end{equation}
	and
	\begin{equation}
	\begin{aligned}
	\sum_{i=0}^{WH-1}|x_i|^2 = \frac{1}{WH}\sum_{k=0}^{WH-1}|X_k|^2,
	\end{aligned}
	\end{equation}
	where $X_k$ denotes the element of $X$, which is the Fourier transform of $x$. Then we have
	\begin{equation}
	\begin{aligned}
	\|\mathcal{F}(h(x))\|\leq\|\mathcal{F}(x)\|,
	\end{aligned}
	\label{theorem11}
	\end{equation}
	which proves Theorem 1.
\end{proof}
	
\subsection{Feature Learning based on Feedback}
	
In this section, we prove the robustness of iffDetector based on the Fourier transform. As mentioned previously, we assume there exists an ideal feature map $n$  for an input image, with $N$ as its corresponding Fourier transform. However, for any input image, we often have a ``polluted" feature map $X^{[0]}$ generated by the CNN backbone defined as
\begin{equation}
	\begin{aligned}
	X^{[0]} = N + \Delta,
	\end{aligned}
	\label{definition-N}
\end{equation}
\noindent
where $\Delta$ denotes the Fourier transform of the noise $\delta$. Correspondingly, we also have an ideal output $y_n$ defined as
\begin{equation}
	\label{definition-YN}
	\begin{aligned}
	y_n = w_1 \otimes n,\\
	Y_N = W_1 \ast N.
	\end{aligned}
\end{equation}

We define the optimization objective function as
\begin{equation}
	\begin{aligned}
	V(Y) = {\left\|Y-Y_N\right\|}^2,
	\end{aligned}
	\label{definition-V}
\end{equation}
\noindent
where $\left\|\cdot\right\|$ denotes the Frobenius norm,  $V$ is a positive definite function and $V(Y)=0$ only when $Y=Y_N$.

Eq.~\ref{definition-V} represents the deviation between the actual output $Y$ and the ideal output $Y_N$. We develop Theorem 2 to show that as the distance between $X^{[k]}$ and  $N$ increases, the distance between $Y^{[k+1]}$ and $Y_N$ is reduced by IFF, which guarantees the stability of our iffDetector. 

\begin{theorem}
In iffDetector, there always exists $\epsilon>0$, such that when $\|X^{[k]}-N\| > \epsilon$, it follows that
	\begin{equation}
	\begin{aligned}
	V(Y^{[k+1]})\le V(Y^{[k]}),
	\end{aligned}
	\end{equation}
    where $V(\cdot)$ is defined in Eq.~\ref{definition-V}.
    \label{theorem2}
\end{theorem}

\begin{proof}
	\label{proof4}
	For simplicity, we first define
	\begin{equation}
	\begin{aligned}
	V^{'}(Y^{[k]})&=V(Y^{[k+1]})-V(Y^{[k]})\\
	& = \|Y^{[k+1]}-Y_N\|^2-\|Y^{[k]}-Y_N\|^2.
	\end{aligned}
	\label{V'1}
	\end{equation}
	Plugging Eqs.~\ref{frequency-1}, \ref{frequency-2} and \ref{definition-YN} into $\|Y^{[k+1]}-Y_N\|^2$ results in 
	\begin{equation}
	\begin{aligned}
	V^{'}(Y^{[k]})&=\|W_1\ast (X^{[0]}-N)+W_1\ast W_2\ast H(W_1\ast X^{[k]})\|^2\\
	&\quad -\|Y^{[k]}-Y_N\|^2.
	\end{aligned}
	\end{equation}
	Besides, it follows that
	\begin{equation}
	\label{bczhangtem}
	\begin{aligned}
	-\|Y^{[k]}-Y_N\|^2&\leq-(\|Y^{[k]}\|-\|Y_N\|)^2\\
	&=-(\|Y^{[k]}\|^2-2\|Y^{[k]}\|\|Y_N\|+\|Y_N\|^2).
	\end{aligned}
	\end{equation}
	 Based on Lemma 2, Eqs. \ref{definition-N}, \ref{definition-YN}, \ref{bczhangtem} and Theorem 1,  we have
	\begin{equation}
	\begin{aligned}
	V'(Y^{[k]})&\leq \|W_1\|^2\|\Delta\|^2
	\quad + 2\|W_1\|^2 \|\Delta\|\|W_2\|\|H(W_1\ast X^{[k]})\|\\
	&\quad + \|W_1\|^2\|W_2\|^2\|H(W_1\ast X^{[k]})\|^2 - \|W_1\|^2\|X^{[k]}\|^2 \\
	&\quad + 2\|W_1\|^2\|X^{[k}]\|\|N\| - \|W_1\|^2\|N\|^2\\
	&\leq \|W_1\|^2\|\Delta\|^2 + 2\|W_1\|^2\|\Delta\|\|W_2\|\|W_1\|\|X^{[k]}\|\\
	&\quad + \|W_1\|^4\|W_2\|^2\|X^{[k]}\|^2 - \|W_1\|^2\|X^{[k]}\|^2\\
	&\quad + 2\|W_1\|^2\|X^{[k]}\|\|N\| - \|W_1\|^2\|N\|^2\\
	& = \|W_1\|^2(\|W_1\|^2\|W_2\|^2 - 1)\|X^{[k]}\|^2\\
	&\quad + 2(\|W_1\|^2\|N\| + \|W_1\|^3\|\Delta\|\|W_2\|)\|X^{[k]}\|\\
	&\quad +\|W_1\|^2(\|\Delta\|^2-\|N\|^2).
	\end{aligned}
	\label{bczhangcore}
	\end{equation}
	 We note that  $W_1$ and $W_2$ are $3$D filters, however, the convolution operation in CNNs is actually the  $2$D convolution. We can only choose the  $2$D matrix of the maximum norm and then directly use Theorem 1.  By setting
	\begin{equation}
	\begin{aligned}
	\left\{
	\begin{array}{ll}
	A = \|W_1\|^2(\|W_1\|^2\|W_2\|^2-1\|),\\
	B = 2(\|W_1\|^2\|N\| + \|W_1\|^3\|\Delta\|\|W_2\|),\\
	C = \|W_1\|^2(\|\Delta\|^2-\|N\|^2),
	\end{array}
	\right.
	\end{aligned}
	\end{equation}
	we obtain
	\begin{equation}
	\label{bczhang}
	\begin{aligned}
	V^{'}(Y^{[k]}) \leq A\|X^{[k]}\|^2 + B\|X^{[k]}\| + C,
	\end{aligned}
	\end{equation}
	where  $A$, $B$ and $C$ are constants\footnote{For a certain input image and forward propagation, the original feature map $x^{[0]}$, the ideal feature map $N$, filters $W_1$ and $W_2$ are all constants.}. By inserting Eq.~\ref{definition-N} into Eq. \ref{frequency-2}, we have
	\begin{equation}
	\begin{aligned}
	X^{[k]} &= N + [\Delta +W_2\ast H(Y^{[k-1]})]\\
	        &= N + \Delta^{[k]},
	\label{mmy}
	\end{aligned}
	\end{equation}
	where $\Delta^{[k]}$ is a variable and equals $X^{[k]}-N$. Inserting Eq.~\ref{mmy} into Eq.\ref{bczhang}, we obtain 
	\begin{equation}
	\label{A}
	\begin{aligned}
	V^{'}(Y^{[k]}) \leq A\|N + \Delta^{[k]}\|^2 + B\|N + \Delta^{[k]}\| + C,
	\end{aligned}
	\end{equation}
	 where $A \le 0$ is easily satisfied because  $\|W_1\|^2 \|W_2\|^2 < 1$  can be achieved by controlling by $W_2$.  According to the property of a quadratic function, there always exists $\epsilon$, such that when $\|\Delta^{[k]}\|=\|X^{[k]}-N\| > \epsilon$, we have $A\|X^{[k]}\|^2+B\|X^{[k]}\|+C < 0$. That is to say, $V(Y^{[k+1]})\leq V(Y^{[k]})$, which proves Theorem 2.
\end{proof}

Theorem 2 indicates that once the deviation of $X^{[k]}$ from $N$ is too large,  $Y^{[k+1]}$ will move back to $Y_N$ (the ideal output), which proves the stability of the filtering module. Eq.~\ref{definition-V} is a classical Lyapunov function \cite{LyapunovThe}, which links the CNN-based detectors to the control theory that explains the feedback strategy can bring robustness to our iffDetector. 

\section{Experiments}
	
In this section, we present the implementation of IFF for detectors to evaluate the effect of the proposed filtering module. We also compare iffDetector with the counterpart and the state-of-the-art approaches. Experiments are carried out on several datasets, including PASCAL VOC2007, PASCAL VOC2012, and MSCOCO 2017.

\subsection{Implementation Details}
	
IFF is implemented on FreeAnchor \cite{zhang2019freeanchor} and CenterNet \cite{CenterNet2019}, which are the state-of-the-art one-stage anchor-based and anchor-free detectors, respectively. We also test the filtering module on other mainstream detectors including SSD \cite{SSD16}, YOLOv2 \cite{YOLO9000}, and FoveaBox \cite{kong2019foveabox}. 
Since we only modify the head of the detector, the pretrained weights of the backbone can be used which saves training time. The training and test times of iffDetectors are almost the same as their baselines. To select the best number of filtering modules ($M_I$), we test our iffDetector based on FreeAnchor with different $M_I$. To demonstrate the generality of IFF, we carry out experiments on three datasets: COCO 2017, and PASCAL VOC2007 and 2012. We use 8 Tesla V100 GPUs for the experiments on FreeAnchor (with ResNeXt-101 as the backbone) and 4 TITAN Xp GPUs for other experiments. The training schedule and batch size are adjusted according to the linear scaling rule \cite{Goyal_2019}.

\subsection{The Choice of $M_I$}

As we mentioned in Section $2$, the number of filtering modules is flexible. To select an appropriate value of $M_I$, we test iffDetector using FreeAnchor (with ResNet-50 as the backbone) on COCO 2017 validation set with different $M_I$, and the results are shown in Table~\ref{M_I}. We find that having more filtering modules does not guarantee better performance. Furthermore, having more filtering modules means more training and inference time. Thus, we set $M_I$ to 1 in all our experiments. Note that $M_I = 0$ corresponds to the original detector without feedback.
\begin{table}[]
	\begin{center}
		\caption{iffDetector performance with different $M_I$.}
		\label{M_I}
		\begin{tabular}{p{1cm}<{\centering} p{2cm}<{\centering}p{1.6cm}<{\centering}p{1.6cm}<{\centering}p{1.6cm}<{\centering}}
			\specialrule{0.1em}{3pt}{3pt}
			$M_I$ & 0    & 1    & 2  &3 \\
			\specialrule{0.1em}{3pt}{3pt}
			mAP  &38.5 &39.3 &39.2 &39.2\\
			\specialrule{0.1em}{3pt}{3pt}
		\end{tabular}
	\end{center}
\end{table}

\subsection{Model Effect}

Feature maps generated by the backbones may have noise such that some regions without object are falsely activated.
We use SSD with or without IFF trained with VOC2007 to illustrate the advantages of effect of the proposed module.
Our filtering module at inference reuses the heuristic information in the original predictions, successfully denoising the feature maps and making the correctly activated regions more intensive, thus leading to more accurate predictions  (see Fig.~\ref{exp}). When multiple objects of the same class appear very close to each other, the inaccurate activated regions on feature maps usually lead to missing detections. We have $6$ feature maps for prediction, and choose the corresponding kernels of the maximum norm. We calculate $A$ in Eq. \ref{A} and obtain  $A = -0.361, -0.301, -0.226, -0.185, -0.201, -0.149$ in a descent order, which are  consistent with Theorem 2. 
Furthermore, iffDetector achieves a better performance with almost the same inference cost, $e.g.$, SSD operates at $45.45$ frames per second (FPS), while iffDetector achieves $44.62$ FPS.
\begin{figure}
\centering
\includegraphics[height=3.5cm]{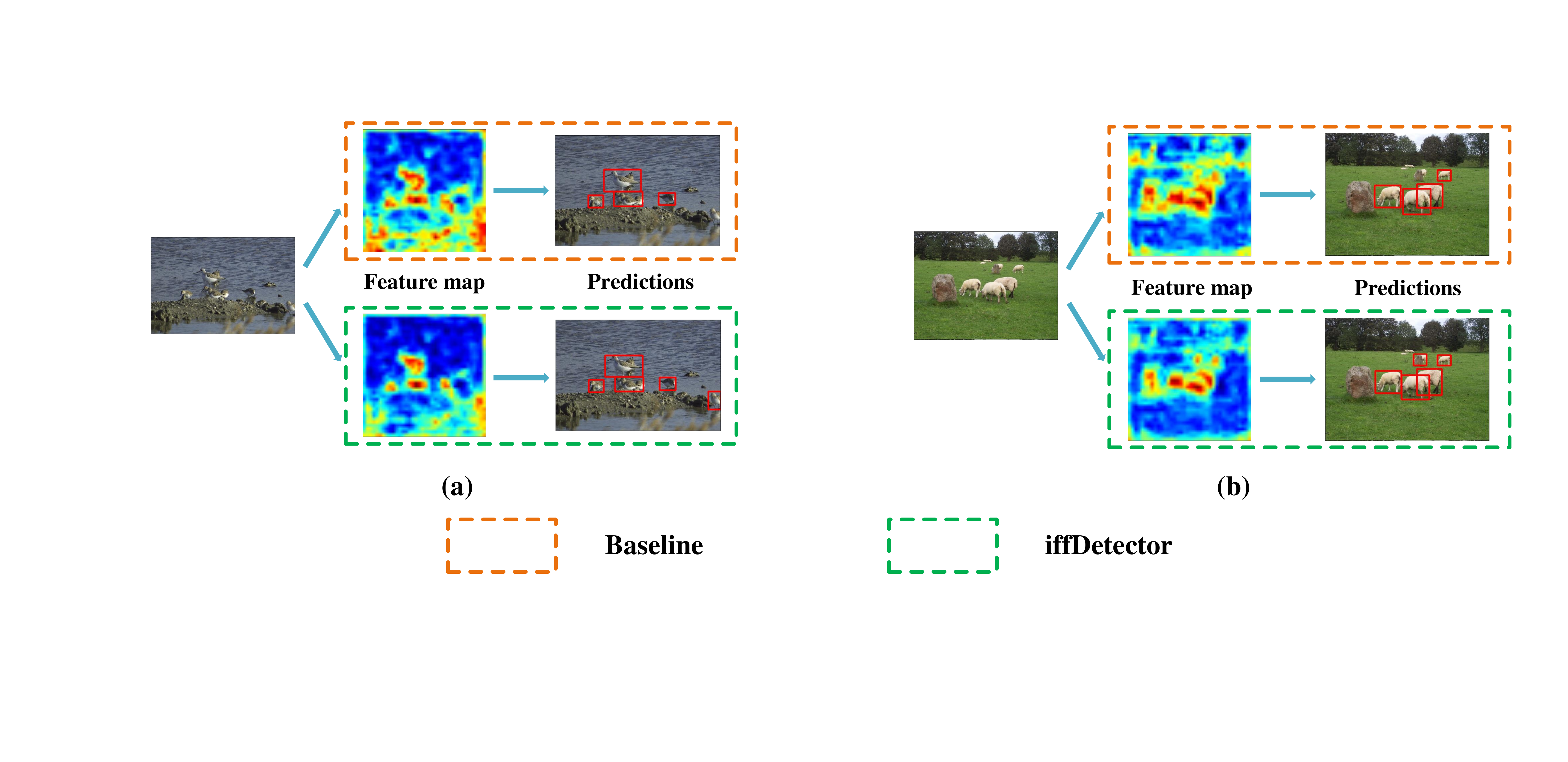}
\caption{Visualization of features w/o IFF: (a) and (b) are the detection results of two input images. To achieve more detailed and direct illustration, we upsample the feature maps to $300 \times 300$, add all the channels, and convert them to heatmaps, where the red regions represent stronger activations.}
\label{exp}
\end{figure}

\begin{figure}[ht]
	\centering
	\subfigure[Background]{\includegraphics[height=2.5cm]{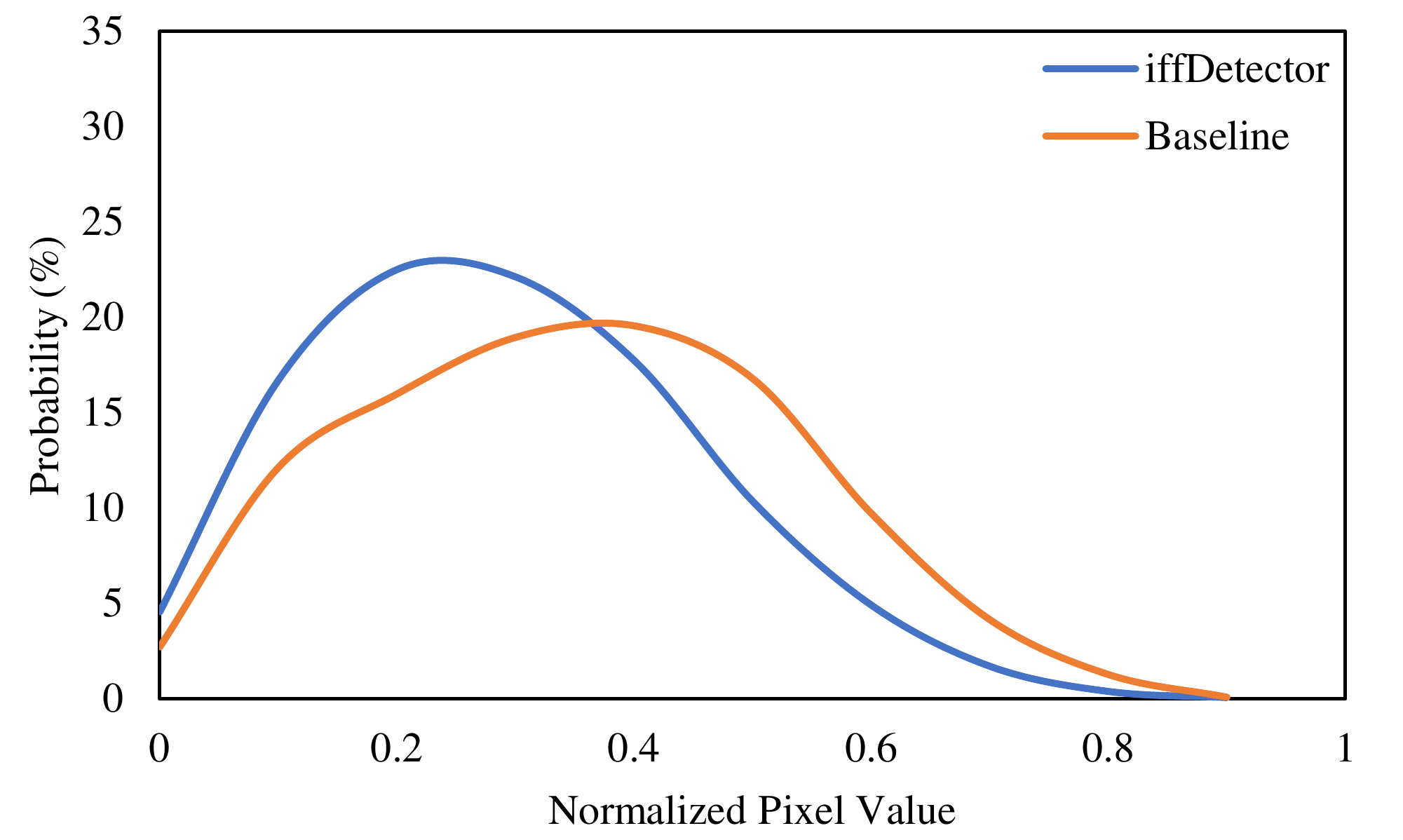}}
	\subfigure[Foreground]{\includegraphics[height=2.5cm]{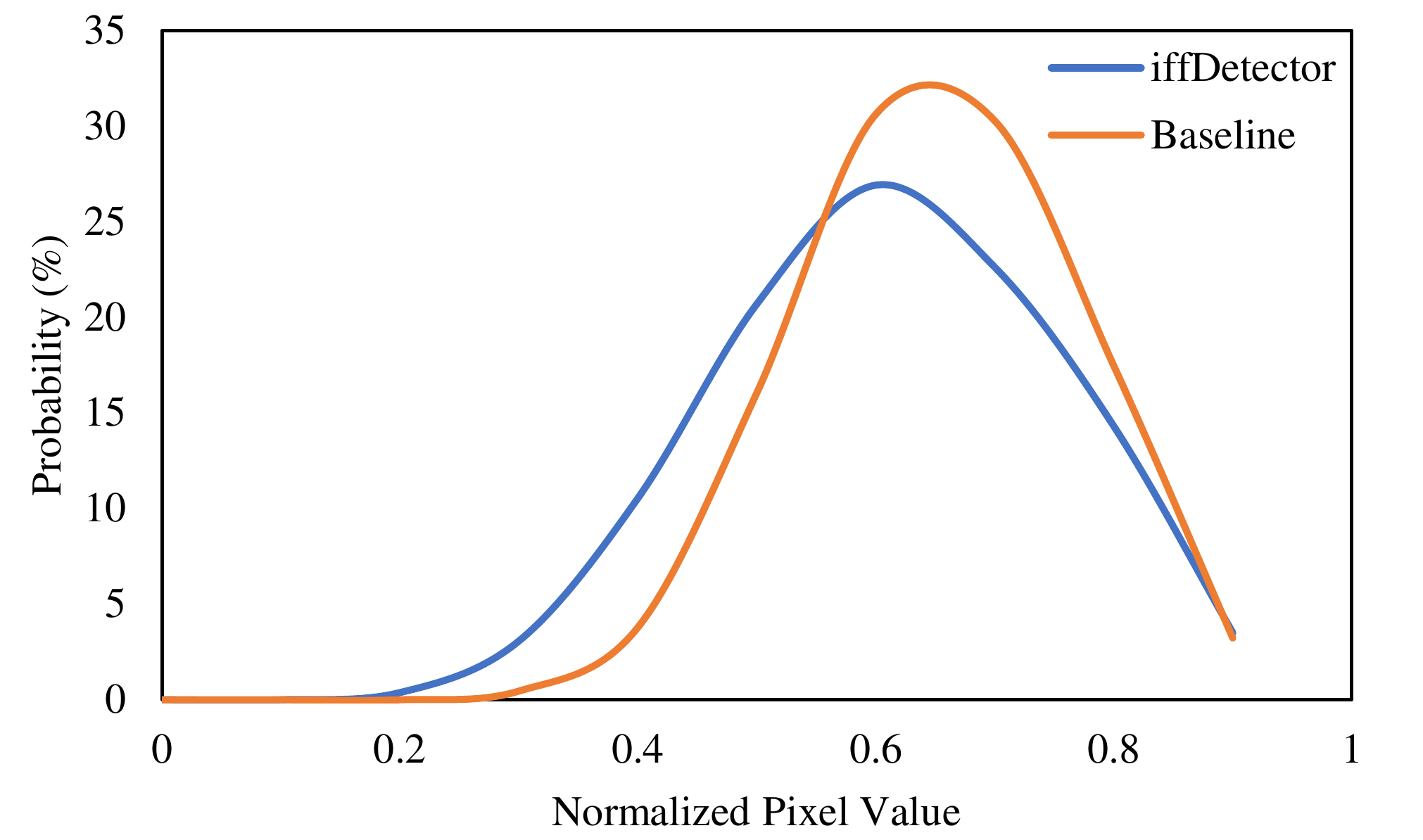}} \\
	\caption{The energy distributions on the feature maps. The baseline is SSD \cite{SSD16}.}
	\label{denoise}
\end{figure}

IFF is designed to filter undesired features enforcing the foreground while depressing the background. To verify this, we show the distribution of pixel values in the background of the feature map, which can be regarded as the energy distribution (see Fig.~\ref{denoise}). To do this, we randomly select 60 images from VOC2007 to obtain two groups of feature maps (original feature maps and feature maps optimized by IFF). For every feature map, we add all the channels together and normalize the pixel values. We then compute the distribution of the pixel values of each group, resulting in two curves corresponding to the two groups.  Fig.~\ref{denoise}(a) shows that the energy of the background obtained by iffDetector is clearly lower than that of the original detector which means IFF successfully depresses the background. Meanwhile,  considering the strongly activated regions of the foreground obtained by iffDetector is more concentrated (see Fig.~\ref{exp}), the energy of the foreground obtained by iffDetector is also a little bit lower than that of the original detector, which is beneficial to object detection, because in most cases the center points of the objects generate the best bounding boxes.
\begin{table}[]
    \begin{center}
        \caption{Detection performance comparison of our iffDetectors and the baselines on PASCAL VOC.}
        \label{baseline-VOC}
        \begin{tabular}{lccc}
            \Xhline{1pt}\noalign{\smallskip}
            Backbone & Detector    & Training    & $mAP$\\
            \noalign{\smallskip}
            \Xhline{1pt}
            \noalign{\smallskip}
            \multirow{2}*{Darknet-19}  & YOLOv2$~416\times 416$
			&\multirow{2}*{VOC07}   & 59.3 \\
			&\textbf{iffDetector} & & \textbf{60.4} \\
			\hline
			\multirow{2}*{VGG-16} & SSD300
			&\multirow{2}*{VOC07+12}  & 77.1 \\
			&\textbf{iffDetector} & & \textbf{78.0} \\
			\Xhline{1pt}
        \end{tabular}
    \end{center}
\end{table}

\begin{table}[]
    \begin{center}
        \caption{Detection performance comparison of our iffDetectors and the baselines on COCO 2017.}
        \label{baseline-COCO}
        \begin{tabular}{llccccccc}
            \Xhline{1pt}\noalign{\smallskip}
            Backbone & Detector  &Parameter Size (M) & $AP$    & $AP_{50}$    & $AP_{75}$    & $AP_S$    & $AP_M$    & $AP_L$\\
            \noalign{\smallskip}
            \Xhline{1pt}
            \noalign{\smallskip}
			\multirow{2}*{ResNet-50} &FoveaBox &36.5  &36.7 &56.6 &39.2 &20.3 &39.9 &45.3\\
			&\textbf{iffDetector}   &36.6 ($\uparrow 0.5\%$)       &\textbf{37.6} &\textbf{57.4} &\textbf{40.4} &\textbf{21.2} &\textbf{40.8} &\textbf{46.7}\\
			
			\hline
			\multirow{2}*{ResNet-50} &FreeAnchor  &33.8 &38.7 &57.3 &41.6 &20.2 &41.3 &50.1\\
			&\textbf{iffDetector}  &34.1 ($\uparrow 0.8\%$) &\textbf{39.6}&\textbf{58.5} &\textbf{42.7} &\textbf{21.3} &\textbf{42.2} &\textbf{50.6}\\	
			
			\hline
			\multirow{2}*{Hourglass-104} &CenterNet &191.3 &45.1 &63.9 &49.3 &26.6 &47.1 &57.7\\
			&\textbf{iffDetector}     &191.7 ($\uparrow 0.2\%$)     &\textbf{45.8} &\textbf{64.8} &\textbf{50.0} &\textbf{27.7} &\textbf{47.8} &\textbf{58.6}\\
			
			\hline
			\multirow{2}*{ResNeXt-101} &FreeAnchor &91.5  &47.3 &66.3 &51.5 &30.6 &50.4 &59.0\\
			&\textbf{iffDetector}    &93.2 ($\uparrow 1.8\%$)      &\textbf{47.9} &\textbf{66.7} &\textbf{52.2} &\textbf{30.7} &\textbf{50.7} &\textbf{59.9}\\
			\Xhline{1pt}
        \end{tabular}
    \end{center}
\end{table}

\subsection{Detection Performance}
	
We compare the proposed iffDetector with multiple baselines, including anchor-based approaches, YOLOv2 \cite{YOLO9000}, SSD \cite{SSD16}, and Freeanchor \cite{zhang2019freeanchor}, and anchor-free approaches, Foveabox \cite{kong2019foveabox} and CenterNet \cite{CenterNet2019}. More experimental results are also provided in the supplementary material, including training curves,  result analysis and other implementation details.

For YOLOv2 and SSD, we upgrade each with our IFF by setting the size of the filters ($w_2$ in Fig.~\ref{architecture}) to $1\times 1$ and test if the iffDetectors outperform the baselines. As shown in Table~\ref{baseline-VOC}, the IFF module works very well on the PASCAL VOC datasets. For other baselines, considering their backbones are large, we set the size of the filters to $3\times 3$ and test the iffDetectors on COCO 2017 (see Table~\ref{baseline-COCO}). Again, IFF consistently improves the baselines on all evaluation metrics with less than a 2\% increase in the number of parameters. Note that all the performance gains are achieved with negligible extra cost of training and inference time. More importantly, compared to the most recent FreeAnchor (on ResNeXt-101), iffDetector respectively improves the $AP$ and $AP_L$ up to 0.6\% and 0.9\%, which are significant margins in terms of the challenging object detection task.

iffDetector is also compared with other state-of-the-art detectors in Table~\ref{SOTA} using standard settings,  ResNet as the backbone. Note that iffDetector$^\ast$ is acquired by updating the corresponding FreeAnchor detectors. Besides, Advanced settings use the jitter over scales of $\{640, 672, 704,  736, 768, 800\}$ during training on the ResNeXt-32x8d-101 backbone. The experiments show that iffDetector outperform both the anchor-based counterparts (RetinaNet \cite{FocalLoss17} and FreeAnchor \cite{zhang2019freeanchor}) and the anchor-free approaches (FoveaBox \cite{kong2019foveabox}, CornerNet \cite{CornerNet2018} and CenterNet \cite{CenterNet2019}). With less than a 2\% increase in the number of parameters and almost the same training and inference time,  IFF upgrades the state-of-the-art detectors with significant margins.
\begin{table}[]
    \begin{center}
        \caption{Performance comparison with state-of-the-art detectors.}
        \label{SOTA}
        \begin{tabular}{llccccccc}
            \Xhline{1pt}\noalign{\smallskip}
            Detector & Backbone  &Iter. & $AP$   &$AP_{50}$    &$AP_{75}$    &$AP_S$    &$AP_M$    &$AP_L$\\
            \noalign{\smallskip}
            \Xhline{1pt}
            \noalign{\smallskip}
            RetinaNet \cite{FocalLoss17}& ResNet-50 &70k     &35.7    &55.0         &38.5         &19.9      &38.9      &46.3  \\
		    FoveaBox \cite{kong2019foveabox}& ResNet-50 &70k     &36.7    &56.6         &39.2      &20.3       &39.9     &45.3     \\
			FreeAnchor \cite{zhang2019freeanchor}&ResNet-50&90k      &38.7    &57.3         &41.6         &20.2      &41.3      &50.1  \\
			\textbf{iffDetector$^\ast$(ours)}&ResNet-50& 90k    &\textbf{39.6}&\textbf{58.5} &\textbf{42.7} &\textbf{21.3} &\textbf{42.2} &\textbf{50.6}\\
			\hline
			CornerNet \cite{CornerNet2018}&Hourglass-104&500k      &40.6    &56.4         &43.2         &19.1      &42.8      &54.3  \\
			RetinaNet \cite{FocalLoss17}&ResNeXt-101&135k      &40.8    &61.1         &44.1         &24.1      &44.2      &51.2  \\
			FoveaBox \cite{kong2019foveabox}&ResNeXt-101&135k      &42.1    &61.9         &45.2         &24.9      &46.8      &55.6  \\
			FSAF \cite{zhu2019feature}&ResNeXt-101 &180k     &44.6    &65.2         &48.6     &29.7      &47.1      &54.6\\
			Cascade R-CNN \cite{cascade18}&ResNeXt-101 &180k     &44.9    &63.7         &48.9     &25.9      &47.7      &57.1\\
			CenterNet \cite{CenterNet2019}&Hourglass-104&480k      &45.1 &63.9 &49.3         &26.6   &47.1       &57.7  \\
			FreeAnchor \cite{zhang2019freeanchor}&ResNeXt-101&180k       &47.3    &66.3         &51.5         &30.6      &50.4      &59.0  \\
			\textbf{iffDetector$^\ast$(ours)}&ResNeXt-101&180k      &\textbf{47.9} &\textbf{66.7} &\textbf{52.2} &\textbf{30.7} &\textbf{50.7} &\textbf{59.9}\\
			\Xhline{1pt}
        \end{tabular}
    \end{center}
\end{table}
\vspace{-1.7em}

\section{Conclusion}
	
We have proposed an elegant and effective approach referred to as Inference-aware Feature Filtering (IFF), for visual object detection. The proposed IFF upgrades both anchor-based and anchor-free frameworks to new detectors with a filtering ability by introducing a feedback architecture to the inference procedure and  achieves more accurate and robust predictions. With IFF implemented, we have significantly improved the performance of object detection over  baseline detectors. The underlying effect is that IFF denoises the feature maps and facilitates generating new feature maps with more precise and intensive activated regions, which lead to improved detection results. 
Our IFF approach provides a fresh insight to the visual object detection problem.


\clearpage
%
%
\bibliographystyle{splncs04}
\bibliography{reference}

\section*{Appendix}
	
\begin{lemma}
	\label{lemma1}
	According to Parseval's Theorem, we have 
	\begin{equation}
	\begin{aligned}
	\sum_{n=0}^{N-1}|x_n|^2=\frac{1}{N}\sum_{k=0}^{N-1}|X_k|^2,
	\end{aligned}
	\end{equation}
	where $x_n$ is the signal in the spatial domain, and $X_k=\mathcal{F}(x_n)$. 
\end{lemma}

\begin{lemma}
	\label{lemma2}
	Any complex matrices X and Y satisfy the following inequality
	\begin{equation}
	\begin{aligned}
	\|X\pm Y\|^2&\leq\|X\|^2+2\|X\|\|Y\|+\|Y\|^2,
	\end{aligned}
	\end{equation}
	where $\|\cdot\|$ denotes any matrix norm.
\end{lemma}
\end{document}